\documentclass[11pt]{article}
\usepackage[a4paper, total={6.5in, 9in}]{geometry}

\usepackage[titletoc,title]{appendix}

\usepackage{url}            % simple URL typesetting
\usepackage{booktabs}       % professional-quality tables
\usepackage{amsfonts}       % blackboard math symbols
\usepackage{nicefrac}       % compact symbols for 1/2, etc.
\usepackage{microtype}      % microtypography

\usepackage{bm}
\usepackage{amssymb}
\usepackage{amsthm}
\usepackage{amsmath}

\usepackage{times}
\usepackage{graphicx}
\usepackage{subfigure}

\usepackage{algorithm}
\usepackage{algorithmic}

\usepackage{multirow}
\usepackage{setspace}

\usepackage[numbers]{natbib}

\newtheorem{lem}{Lemma}
\newtheorem{prop}{Proposition}

\title{Combinatorial Topic Models using Small-Variance Asymptotics}
\date{}

{\tiny
\author{
  Ke Jiang \\
  Dept. of Computer Science and Engineering\\
  Ohio State University\\
  \texttt{jiang.454@osu.edu}
  \and
  Suvrit Sra \\
  Lab. for Information and Decision Systems \\
  Massachusetts Institute of Technology \\
  \texttt{suvrit@mit.edu}
  \and
  Brian Kulis \\
  Dept. of Electrical \& Computer Engineering and Dept. of Computer Science \\
  Boston University \\
  \texttt{bkulis@bu.edu}
}
}

\begin{document}

\maketitle

\begin{abstract}
Topic models have emerged as fundamental tools in unsupervised machine learning.  Most modern topic modeling algorithms take a probabilistic view and derive inference algorithms based on Latent Dirichlet Allocation (LDA) or its variants.  In contrast, we study topic modeling as a combinatorial optimization problem, and propose a new  objective function derived from LDA by passing to the small-variance limit. We minimize the derived objective by using ideas from combinatorial optimization, which results in a new, fast, and high-quality topic modeling algorithm.  In particular, we show that our results are competitive with popular LDA-based topic modeling approaches, and also discuss the (dis)similarities between our approach and its probabilistic counterparts.
%the surprising result that our algorithm can outperform all major LDA-based topic modeling approaches, even when the data are sampled from an LDA model and true hyper-parameters are provided to these competitors. These results make a strong case that topic models need not be limited to a probabilistic view.
\end{abstract}

\section{Introduction}

Topic modeling has long been fundamental to unsupervised learning on large document collections. Though the roots of topic modeling date back to latent semantic indexing~\cite{deerwester_lsi} and probabilistic latent semantic indexing~\cite{hofmann_plsi}, the arrival of Latent Dirichlet Allocation (LDA)~\cite{blei_lda} was a turning point that transformed the community's thinking about topic modeling. LDA led to several followups that address some limitations of the original model~\cite{correlated_lda,spatial_lda}, and also helped pave the way for subsequent advances in Bayesian learning methods, including variational inference methods~\cite{collapsed_vb}, nonparametric Bayesian models~\cite{nested_crp,hierarchical_dp}, among others.

The LDA family of topic models are almost exclusively cast as probabilistic models. Consequently, the vast majority of techniques developed for topic modeling---collapsed Gibbs sampling~\cite{lda_gibbs}, variational methods~\cite{blei_lda,collapsed_vb}, and ``factorization'' approaches with theoretical guarantees~\cite{anandkumar2012spectral,arora2012learning,bansal2014provable}---are centered around performing inference for underlying probabilistic models. By limiting ourselves to a purely probabilistic viewpoint, we may be missing important opportunities grounded in combinatorial thinking. This realization leads us to the central question of this paper: \emph{Can we obtain a combinatorial topic model that competes with LDA?}

%For instance, consider the analogy: Gaussian mixtures are representative probabilistic models, yet often one prefers the simpler and more practical k-means algorithm, which arises in response to the combinatorial problem of clustering. %, which arises from a combinatorial optimization problem.% Indeed, for large-scale data k-means is the workhorse for clustering, and is used in problems such as image feature quantization and other tasks~\cite{vbow_cvpr}.  

% 

We answer this question in the affirmative. In particular, we propose a combinatorial optimization formulation for topic modeling, derived using \emph{small-variance asymptotics} (SVA) on the LDA model. SVA produces limiting versions of various probabilistic learning models, which can then be solved as combinatorial optimization problems. An analogy worth keeping in mind here is how $k$-means solves the combinatorial problem that arises upon letting variances go to zero in Gaussian mixtures. % (and similar connections between probabilistic PCA and PCA~\cite{roweis,probabilistic_pca}, among others) 

SVA techniques have proved quite fruitful recently, e.g., for cluster evolution~\cite{campbell_nips}, hidden Markov models~\cite{sva_hmm}, feature learning~\cite{mad_bayes}, supervised learning~\cite{sva_dpsvm}, hierarchical clustering~\cite{sva_hierarchical_clustering}, and others~\cite{jump_means,dp_space}. A common theme in these examples is that computational advantages and good empirical performance of $k$-means carry over to richer SVA based models. Indeed, in a compelling example, \cite{campbell_nips} demonstrate how a hard cluster evolution algorithm obtained via SVA is orders of magnitude faster than competing sampling-based methods, while still being significantly more accurate than competing probabilistic inference algorithms on benchmark data.  % 

But merely using SVA to obtain a combinatorial topic model does not suffice. We need effective algorithms to optimize the resulting model. Unfortunately, a direct application of greedy combinatorial procedures on the LDA-based SVA model \emph{fails} to compete with the usual probabilistic LDA methods. This setback necessitates a new idea. Surprisingly, as we will see, a local refinement procedure combined with an improved word assignment technique transforms the SVA approach into a competitive topic modeling algorithm.

\textbf{Contributions}. In summary the main contributions of our paper are the following:
\begin{itemize}
\item We perform SVA on the standard LDA model and obtain through it a combinatorial topic model. % (where part of the problem is essentially a facility location problem).  % and is amenable to a greedy optimization procedure.
\item We develop an optimization procedure for optimizing the derived combinatorial model by utilizing local refinement and ideas from the facility location problem.
\item We show how our procedure can be implemented to take $O(NK)$ time per iteration to assign each word token to a topic, where $N$ is the total number of word tokens and $K$ the number of topics.
\item We demonstrate that our approach competes favorably with existing state-of-the-art topic modeling algorithms; in particular, our approach is orders of magnitude faster than sampling-based approaches, with comparable or better accuracy.
%\item  We derive and implement a powerful local refinement algorithm that performs incremental word and mini-topic assignments. Combined with the $k$-means like procedure, this yields an effective topic modeling algorithm. % The role of local search is crucial for our SVA based model because direct greedy optimization \emph{fails} to compete with probabilistic approaches.
%\item Empirical setup; insights. \todo{yet to do}
\end{itemize}

Before proceeding to outline the technical details, we make an important comment regarding evaluation of topic models. The connection between our approach and standard LDA may be viewed analogously to the connection between k-means and a Gaussian mixture model. As such, evaluation is nontrivial; most topic models are evaluated using predictive log-likelihood or related measures. In light of the ``hard-vs-soft'' analogy, a predictive log-likelihood score can be a misleading way to evaluate performance of the k-means algorithm, so clustering comparisons typically focus on ground-truth accuracy (when possible). Due to the lack of available ground truth data, to assess our combinatorial model we must resort to synthetic data sampled from the LDA model to enable meaningful quantitative comparisons; but in line with common practice we also present results on real-world data, for which we use both hard and soft predictive log-likelihoods.

\subsection{Related Work}

%Interest in topic models has grown so rapidly that we cannot hope to do full justice to related work. However, we summarize below some works from key related subareas.
\textbf{LDA Algorithms.} Many techniques have been developed for efficient inference for LDA.  The most popular are perhaps MCMC-based methods, notably the collapsed Gibbs sampler (CGS)~\cite{lda_gibbs}, and variational inference methods~\cite{blei_lda,collapsed_vb}.  Among MCMC and variational techniques, CGS typically yields excellent results and is guaranteed to sample from the desired posterior with sufficiently many samples. Its running time can be slow and many samples may be required before convergence.

Since topic models are often used on large (document) collections, significant effort has been made in scaling up LDA algorithms. One recent example is \citep{li2014reducing} that presents a massively distributed implementation. Such methods are outside the focus of this paper, which focuses more on our new combinatorial model that can quatitatively compete with the probabilistic LDA model. Ultimately, our model should be amenable to fast distributed solvers, and obtaining such solvers for our model is an important part of future work.
%Given its strong empirical performance, obtaining faster algorithms for our model is an important part of future work.

A complementary line of algorithms starts with~\citep{arora2012learning,arora2012practical}, who consider certain separability assumptions on the input data to circumvent NP-Hardness of the basic LDA model. These works have shown performance competitive to Gibbs sampling in some scenarios while also featuring theoretical guarantees. Other recent viewpoints on LDA are offered by~\citep{anandkumar2012spectral,bach2015lda,bansal2014provable}.

\textbf{Small-Variance Asymptotics (SVA).} As noted above, SVA has recently emerged as a powerful tool for obtaining scalable algorithms and objective functions by ``hardening'' probabilistic models. % Recall the link between k-means and Gaussian mixtures, namely that the k-means objective function arises from the negative log-likelihood of the Gaussian mixture model as the variance of each cluster shrinks to zero, and that the k-means algorithm arises from the EM algorithm under the same limiting behavior.
Similar connections are known for instance in dimensionality reduction~\cite{roweis}, multi-view learning, classification~\cite{tong_koller}, and structured prediction~\cite{chang_icml}.  
%Figure~\ref{tab:sva} lists some examples of SVA applied to these problems.
Starting with Dirichlet process mixtures~\cite{kulis_jordan}, one thread of research has considered applying SVA to richer Bayesian nonparametric models. Applications include clustering~\cite{kulis_jordan}, feature learning~\cite{mad_bayes}, %(where, small-variance methods even displayed greater than three orders of magnitude speedup over corresponding sampling-based methods, without sacrificing accuracy), 
evolutionary clustering~\cite{campbell_nips}, infinite hidden Markov models~\cite{sva_hmm}, Markov jump processes~\cite{jump_means}, infinite SVMs~\cite{sva_dpsvm}, and hierarchical clustering methods~\cite{sva_hierarchical_clustering}.  A related thread of research considers how to apply SVA methods when the data likelihood is not Gaussian, which is precisely the scenario under which LDA falls.  In~\citep{jiang_nips}, it is shown how SVA may be applied as long as the likelihood is a member of the exponential family of distributions.  Their work considers topic modeling as a potential application, but does not develop any algorithmic tools, and without these SVA \emph{fails} to succeed on topic models; the present paper fixes this by using a stronger word assignment algorithm and introducing local refinement. % only compares qualitatively to existing methods and does not develop methods beyond the simple iterative approaches developed in~\citep{kulis_jordan}.
%We will see later how crucial it is to develop further algorithmic tools to optimize the SVA objective for LDA.

%\subsection{Combinatorial Optimization}
\textbf{Combinatorial Optimization.} In developing effective algorithms for topic modeling, we will borrow some ideas from the large literature on combinatorial optimization algorithms.  In particular, in the $k$-means community, significant effort has been made on how to improve upon the basic $k$-means algorithm, which is known to be prone to local optima; these techniques include local search methods~\cite{local_search} and good initialization strategies~\cite{kmeans_plus_plus}.  We also borrow ideas from approximation algorithms, most notably algorithms based on the facility location problem~\cite{greedy_fl}.

%%% Local Variables:
%%% mode: latex
%%% TeX-master: "main_nips"
%%% End:

\section{SVA for Latent Dirichlet Allocation}

\label{sec:algo}
We now detail our combinatorial approach to topic modeling. We start with the derivation of the underlying objective function that is the basis of our work. This objective is derived from the LDA model by applying SVA, and contains two terms. The first is similar to the $k$-means clustering objective in that it seeks to assign words to topics that are, in a particular sense, ``close.''  The second term, arising from the Dirichlet prior on the per-document topic distributions, places a penalty on the number of topics per document.

Recall the standard LDA model. We choose topic weights for each document as $\theta_j \sim \mbox{Dir}(\alpha)$, where $j \in \{1, ..., M\}$.  Then we choose word weights for each topic as $\psi_i \sim \mbox{Dir}(\beta)$, where $i \in \{1, ..., K\}$.  Then, for each word $i$ in document $j$, we choose a topic $z_{jt} \sim \mbox{Cat}(\theta_j)$ and a word $w_{jt} \sim \mbox{Cat}(\psi_{z_{jt}})$.
%\begin{itemize}
%\item Choose $\theta_j \sim \mbox{Dir}(\alpha)$, where $j \in \{1, ..., M\}$.
%\item Choose $\psi_i \sim \mbox{Dir}(\beta)$, where $i \in \{1, ..., K\}$.
%\item For each word $t$ in document $j$:
%\begin{itemize}
%\item Choose a topic $z_{jt} \sim \mbox{Cat}(\theta_j)$.
%\item Choose a word $w_{jt} \sim \mbox{Cat}(\psi_{z_{jt}})$.
%\end{itemize}
%\end{itemize}
Here $\alpha$ and $\beta$ are scalars (i.e., we are using a symmetric Dirichlet distribution).  Let $\mathbf{W}$ denote the vector of all words in all documents, $\mathbf{Z}$ the topic indicators of all words in all documents, $\bm{\theta}$ the concatenation of all the $\theta_j$ variables, and $\bm{\psi}$ the concatenation of all the $\psi_i$ variables.  Also let $N_j$ be the total number of word tokens in document $j$.  The $\theta_j$ vectors are each of length $K$, the number of topics.  The $\psi_i$ vectors are each of length $D$, the size of the vocabulary. We can write down the full joint likelihood $p(\mathbf{W},\mathbf{Z},\bm{\theta},\bm{\psi} | \alpha, \beta)$ of the model in the factored form
\begin{displaymath}
\prod_{i=1}^K p(\psi_i | \beta) \prod_{j=1}^M p(\theta_j | \alpha) \prod_{t=1}^{N_j} p(z_{jt} | \theta_j)p (w_{jt} | \psi_{z_{jt}}),
\end{displaymath}
where each of the probabilities is as specified by the LDA model. Following standard LDA manipulations, we can eliminate variables to simplify inference by integrating out $\bm{\theta}$ to obtain
\begin{equation}
\label{eq:1}
p(\mathbf{Z},\mathbf{W}, \bm{\psi} | \alpha, \beta) = \int_{\bm{\theta}} p(\mathbf{W},\mathbf{Z},\bm{\theta},\bm{\psi} | \alpha, \beta) d \bm{\theta}.
\end{equation}
%When integrating out a Dirichlet distribution times a multinomial distribution, i.e. $\int p(z| \theta) p(\theta | \alpha) d \theta$, where $p(z | \theta)$ is multinomial and $p(\theta | \alpha)$ is Dirichlet, one obtains the \textit{Dirichlet-multinomial} distribution.  We therefore obtain $p(\bm{Z},\bm{W}, \bm{\psi} | \alpha, \beta) = $
After integration and some simplification, \eqref{eq:1} becomes %this obtain $p(\mathbf{Z},\mathbf{W}, \bm{\psi} | \alpha, \beta) =$
\begin{equation}
%&  & \prod_{i=1}^K p(\psi_i | \beta) \prod_{j=1}^M \prod_{t=1}^{N_j} p(w_{jt} | \psi_{z_{jt}}) \int_{\bm{\theta}} \prod_{j=1}^M p(\theta_j | \alpha) \prod_{t=1}^{N_j} p(z_{jt} | \theta_j) d \bm{\theta}\\
\bigg [\prod_{i=1}^K p(\psi_i | \beta) \prod_{j=1}^M \prod_{t=1}^{N_j} p(w_{jt} | \psi_{z_{jt}}) \bigg ] \times \bigg [ \prod_{j=1}^M \frac{\Gamma(\alpha K)}{\Gamma(\sum_{i=1}^K n_{j \cdot}^i + \alpha K)} \prod_{i=1}^K \frac{\Gamma(n_{j \cdot}^i + \alpha)}{\Gamma(\alpha)} \bigg ].
\label{eqn:loglikelihood}
\end{equation}
Here $n_{j \cdot}^i$ is the number of word tokens in document $j$ assigned to topic $i$. Now, following~\cite{mad_bayes}, we can obtain the SVA objective by taking the (negative) logarithm of this likelihood and letting the variance go to zero.  Given space considerations, we will summarize this derivation; full details are available in Appendix A.

Consider the first bracketed term of~\eqref{eqn:loglikelihood}. Taking logs yields a sum over terms of the form $\log p(\psi_i|\beta)$ and terms of the form $\log p(w_{jt} | \psi_{z_{jt}})$.  Noting that the latter of these is a multinomial distribution, and thus a member of the exponential family, we can appeal to the results in~\citep{banerjee_bregman,jiang_nips} to introduce a new parameter for scaling the variance. In particular, we can write $p(w_{jt} | \psi_{z_{jt}})$ in its \emph{Bregman divergence} form $\exp(-\mbox{KL}(\tilde{w}_{jt},\psi_{z_{jt}}))$, 
where KL refers to the discrete KL-divergence, and $\tilde{w}_{jt}$ is an indicator vector for the word at token $w_{jt}$.  It is straightforward to verify that $\mbox{KL}(\tilde{w}_{jt},\psi_{z_{jt}}) = -\log \psi_{z_{jt},w_{jt}}$. Next, introduce a new parameter $\eta$ that scales the variance appropriately, and write the resulting distribution as proportional to $\exp(-\eta \cdot \mbox{KL}(\tilde{w}_{jt},\psi_{z_{jt}}))$.  As $\eta \to \infty$, the expected value of the distribution remains fixed while the variance goes to zero, exactly what we require.

After this, consider the second bracketed term of~\eqref{eqn:loglikelihood}.  We scale $\alpha$ appropriately as well; this ensures that the hierarchical form of the model is retained asymptotically.  In particular, we write $\alpha = \exp(-\lambda \cdot \eta)$.  After some manipulation of this distribution, we can conclude that the negative log of the Dirichlet multinomial term becomes asymptotically $ \eta \lambda (K_{j+} - 1)$, where $K_{j+}$ is the number of topics $i$ in document $j$ where $n_{j \cdot}^i > 0$, i.e., the number of topics currently used by document $j$.  (The maximum value for $K_{j+}$ is $K$, the total number of topics.)  To formalize, let $f(x) \sim g(x)$ denote that $f(x) / g(x) \to 1$ as $x \to \infty$. Then we have the following (see Appendix A for a proof):
\begin{lem}
Consider the likelihood
\begin{displaymath}
p(\bm{Z} | \alpha) = \bigg [ \prod_{j=1}^M \frac{\Gamma(\alpha K)}{\Gamma(\sum_{i=1}^K n_{j \cdot}^i + \alpha K)} \prod_{i=1}^K \frac{\Gamma(n_{j \cdot}^i + \alpha)}{\Gamma(\alpha)} \bigg ].
\end{displaymath}
If $\alpha = \exp(-\lambda \cdot \eta)$, then asymptotically as $\eta \to \infty$, the negative log-likelihood satisfies
\begin{displaymath}
-\log p(\bm{Z} | \alpha) \sim \eta \lambda \sum\nolimits_{j=1}^M(K_{j+} - 1).
\end{displaymath}
\end{lem}

Now we put the terms of the negative log-likelihood together.  The $-\log p(\psi_i | \beta)$ terms vanish asymptotically since we are not scaling $\beta$ (see the note below on scaling $\beta$).  Thus, the remaining terms in the SVA objective are the ones arising from the word likelihoods and the Dirichlet-multinomial.
%which, after applying a negative logarithm, become
%\begin{displaymath}
%-\sum_{j=1}^M \sum_{t=1}^{N_j} \log p(w_{jt} | \psi_{z_{jt}}).
%\end{displaymath}
Using the Bregman divergence representation with the additional $\eta$ parameter, we conclude that the negative log-likelihood asymptotically yields the following:
\begin{displaymath}
-\log p(\mathbf{Z},\mathbf{W}, \bm{\psi} | \alpha, \beta)
\sim
\eta \bigg [ \sum_{j=1}^M \sum_{t=1}^{N_j} \mbox{KL}(\tilde{w}_{jt},\psi_{z_{jt}}) + \lambda \sum_{j=1}^M (K_{j+}-1) \bigg ],
\end{displaymath}
which leads to our final objective function
\begin{equation}
\min_{\bm{Z},\bm{\psi}} \bigg ( \sum_{j=1}^M \sum_{t=1}^{N_j} \mbox{KL}(\tilde{w}_{jt},\psi_{z_{jt}}) + \lambda \sum_{j=1}^M K_{j+} \bigg ).
\label{eqn:obj}
\end{equation}
We remind the reader that $\mbox{KL}(\tilde{w}_{jt},\psi_{z_{jt}}) = -\log \psi_{z_{jt},w_{jt}}$.  Thus, we obtain a $k$-means-like term that says that all words in all documents should be ``close'' to their assigned topic in terms of KL-divergence, but that we should also not have too many topics represented in each document.

\textbf{Note}. We did not scale $\beta$ to obtain a simple objective with only one parameter (other than the total number of topics), but let us say a few words about scaling $\beta$.  A natural approach is to further integrate out $\bm{\psi}$ of the joint likelihood, as is done in the collapsed Gibbs sampler.  One would obtain additional Dirichlet-multinomial distributions, and properly scaling as discussed above would yield a simple objective that places penalties on the number of topics per document as well as the number of words in each topic.  Optimization would then be performed with respect to the topic assignment matrix.  Future work will consider effectiveness of such an objective function for topic modeling.  
%I have not worked out what happens if we try to scale $\beta$ on the above likelihood, but another option is to integrate out $\bm{\psi}$, as is done with the collapsed Gibbs sampler for LDA.  This would yield an analogous Dirichlet-multinomial distribution; in detail, we would obtain
%\begin{eqnarray*}
%p(\bm{Z},\bm{W} | \alpha, \beta) & = & \int_{\bm{\psi}} \int_{\bm{\theta}} p(\bm{W},\bm{Z},\bm{\theta},\bm{\psi} | \alpha, \beta) d \bm{\psi} d \bm{\theta}\\
%& = &\bigg [ \prod_{i=1}^K \frac{\Gamma(\beta V)}{\Gamma (\sum_{r=1}^V n_{\cdot r}^i + \beta V)} \prod_{r=1}^V \frac{\Gamma(n_{\cdot r}^i + \beta)}{\Gamma(\beta)} \bigg ] \bigg [ \prod_{j=1}^M \frac{\Gamma(\alpha K)}{\Gamma(\sum_{i=1}^K n_{j \cdot}^i + \alpha K)} \prod_{i=1}^K \frac{\Gamma(n_{j \cdot}^i + \alpha)}{\Gamma(\alpha)} \bigg ],
%\end{eqnarray*}
%where $V$ is the size of the vocabulary and $n_{\cdot r}^i$ is the number of instances of word $r$ in topic $i$ across all documents.  Performing asymptotics with a similar scaling of $\beta$ as we had for $\alpha$, we would obtain a penalty term that accounts for the number of words assigned to each topic.  In other words, we would be left only with two penalties, one on the number of topics per document and the other on the number of words per topic.  (NB: At one point I thought about this objective, and if I recall correctly I concluded that it has a trivial solution, and is thus not interesting.)

\section{Algorithms}

\label{sec:algorithms}
With our combinatorial objective in hand, we are ready to develop algorithms that optimize it.  In particular, we discuss a locally-convergent algorithm similar to $k$-means and the hard topic modeling algorithm \cite{jiang_nips}. Then, we introduce two more powerful techniques: (i) a word-level assignment method that arises from connections between our proposed objective function and the facility location problem; and (ii) an incremental topic refinement method that is inspired by local-search methods developed for $k$-means.  Despite the apparent complexity of our algorithms, we show that the per-iteration time matches that of the collapsed Gibbs sampler (while empirically converging in just a few iterations, as opposed to the thousands typically required for Gibbs sampling).

%\vspace{-.2cm}
%\subsection{The Basic Batch Algorithm}
%\vspace{-.2cm}
%\begin{algorithm}[tb]
%  \caption{Basic Batch Algorithm}
%  \label{alg:basic}
%\begin{algorithmic}
%   \STATE {\bfseries Input:} Words: $\mathbf{W}$, Number of topics: $K$, Topic penalty: $\lambda$
%   \STATE Initialize $\mathbf{Z}$ and topic vectors $\psi_1, ..., \psi_K$.
%   \STATE Compute initial objective function~\eqref{eqn:obj} using $\mathbf{Z}$ and $\bm{\psi}$.
%   \REPEAT
%   \STATE \textit{//Update assignments:}
%   \FOR{every word token $t$ in every document $j$}
%   \STATE Compute distance $d(j,t,i)$ to topic $i$: $-\log(\psi_{i,w_{jt}})$.
%   \STATE If $z_{jt} \neq i$ for all tokens $t$ in document $j$, add $\lambda$ to $d(j,t,i)$.
%   \STATE Obtain assignments via $Z_{jt} = \mbox{argmin}_i d(j,t,i)$.
%   \ENDFOR
%   \STATE \textit{//Update topic vectors:}
%   \FOR{every element $\psi_{iu}$}
%   \STATE $\psi_{iu} = $ \# occ. of word $u$ in topic $i$ / total \# of word tokens in topic $i$.
%   \ENDFOR
%   \STATE Recompute objective function~\eqref{eqn:obj} using updated $\mathbf{Z}$ and $\bm{\psi}$.
%   \UNTIL{no change in objective function.}
%   \STATE {\bfseries Output:} Assignments $\mathbf{Z}$.
%\end{algorithmic}
%\end{algorithm}

We first describe a basic iterative algorithm for optimizing the combinatorial hard LDA objective derived in the previous section (see Appendix A for pseudo-code).  The basic algorithm follows the $k$-means style---we perform alternate optimization by first minimizing with respect to the topic indicators for each word (the $\mathbf{Z}$ values) and then minimizing with respect to the topics (the $\bm{\psi}$ vectors).  

%See Algorithm \ref{alg:basic} for pseudo-code.  
Consider first minimization with respect to $\bm{\psi}$, with $\mathbf{Z}$ fixed.  In this case, the penalty term of the objective function for the number of topics per document is not relevant to the minimization.  Therefore the minimization can be performed in closed form by computing means based on the assignments, due to known properties of the KL-divergence; 
see Proposition 1 of~\cite{banerjee_bregman}.  In our case, the topic vectors will be computed as follows: entry $\psi_{iu}$ corresponding to topic $i$ and word $u$ will simply be equal to the number of occurrences of word $u$ assigned to topic $i$ normalized by the total number of word tokens assigned to topic $i$.

Next consider minimization with respect to $\mathbf{Z}$ with fixed $\bm{\psi}$. We follow a strategy similar to DP-means~\cite{kulis_jordan}.  In particular,
 %This is less straightforward due to the presence of the penalty terms for the number of topics per document.  In fact, unlike the assignment step for $k$-means, this assignment step can be shown to be NP-hard.  However, the assignment problem turns out to be an instance of a non-parametric clustering known as DP-means \cite{kulis_jordan}, which considers a $k$-means objective with a penalty on the number of clusters. We can follow a similar strategy for assignments here.  In particular, 
  we compute the KL-divergence between each word token $w_{jt}$ and every topic $i$ via $-\log (\psi_{i,w_{jt}})$.  Then, for any topic $i$ that is not currently occupied by any word token in document $j$, i.e., $z_{jt} \neq i$ for all tokens $t$ in document $j$, we \textit{penalize} the distance by $\lambda$.  Next we obtain new assignments by reassigning each word token to the topic corresponding to its smallest divergence (including any penalties).  We continue this alternating strategy until convergence.
%no further changes are made to the assignments from one iteration to the next, or when the change in the objective function falls below some threshold.
The running time of the batch algorithm can be shown to be $O(NK)$ per iteration, where $N$ is the total number of word tokens and $K$ is the number of topics.  
%This is because each word token must be compared to every topic, but the resulting comparison can be done in constant time.  Updating topics is performed by maintaining a count of the number of occurrences of each word in each topic, which also runs in $O(NK)$ time.  Note that the collapsed Gibbs sampler runs in $O(NK)$ time per iteration, and thus has a comparable running time per iteration.
One can also show that this algorithm is guaranteed to converge to a local optimum, similar to $k$-means and DP-means.  
%The argument follows along similar lines to $k$-means and DP-means, namely that each updating step cannot increase the objective function.  In particular, the update on the topic vectors must improve the objective function since the means are known to be the best representatives for topics based on the results of~\cite{banerjee_bregman}.  The assignment step must decrease the objective since we only re-assign if the distance goes down.  Further, we only re-assign to a topic that is not currently used by the document if the distance is more than $\lambda$ greater than the distance to the current topic, thus accounting for the additional $\lambda$ that must be paid in the objective function.

\subsection{Improved Word Assignments}

%The basic algorithm has the advantage that it achieves local convergence, and can be implemented to run efficiently.  However, as we will see in our experiments, this algorithm is also prone to falling into poor local optima, as it often the case with $k$-means-style algorithms.  In this section, we discuss and analyze an alternative assignment technique for $\mathbf{Z}$, which may be used as an initialization to the locally convergent hard LDA algorithm or to replace the basic assignment strategy of hard LDA.
The basic algorithm has the advantage that it achieves local convergence. However, it is quite sensitive to initialization, analogous to standard k-means. 
%Actually, it will have almost no control over the number of topics used by each document due to its simple handling of the penalty $\lambda$, when initialized randomly.
In this section, we discuss and analyze an alternative assignment technique for $\mathbf{Z}$, which may be used as an initialization to the locally-convergent basic algorithm or to replace it completely.

\begin{algorithm}[tb]
   \caption{Improved Word Assignments for $\mathbf{Z}$}
   \label{alg:ufl}
\begin{algorithmic}
   \STATE {\bfseries Input:} Words: $\mathbf{W}$, Number of topics: $K$, Topic penalty: $\lambda$, Topics: $\bm{\psi}$
   \FOR{every document $j$}
   %\STATE Initialize the set of topics for $j$ to be empty: $F = \phi$.
   \STATE Let $f_i = \lambda$ for all topics $i$.
   \STATE Initialize all word tokens to be unmarked.
   \WHILE{there are unmarked tokens}
   \STATE Pick the topic $i$ and set of unmarked tokens $T$ that minimizes~\eqref{eqn:fl}.
%   \begin{displaymath}
 %  \frac{f_i + \sum_{t \in T} \mbox{KL}(\tilde{w}_{jt},\psi_i)}{|T|}.
 %  \end{displaymath}
   \STATE Let $f_i = 0$ and mark all tokens in $T$.
   \STATE Assign $z_{jt} = i$ for all $t \in T$.
   \ENDWHILE
   \ENDFOR
   \STATE {\bfseries Output:} Assignments $\mathbf{Z}$.
\end{algorithmic}
\end{algorithm}

Algorithm~\ref{alg:ufl} details the alternate assignment strategy for tokens.  The inspiration for this greedy algorithm arises from the fact that we can view the assignment problem for $\mathbf{Z}$, given $\bm{\psi}$, as an instance of the uncapacitated facility location (UFL) problem \cite{greedy_fl}.  Recall that the UFL problem aims to open a set of facilities from a set $F$ of potential locations.  Given a set of clients $D$, a distance function $d: D \times F \rightarrow \mathbb{R}_+$, and a cost function $f: F \rightarrow \mathbb{R}_+$ for the set $F$, the UFL problem aims to find a subset $S$ of $F$ that minimizes $\sum_{i \in S} f_i + \sum_{j \in D} (\min_{i \in S} d_{ij})$.

To map UFL to the assignment problem in combinatorial topic modeling, consider the problem of assigning word tokens to topics for some fixed document $j$.  The topics correspond to the facilities and the clients correspond to word tokens.  Let $f_i = \lambda$ for each facility, and let the distances between clients and facilities be given by the corresponding KL-divergences as detailed earlier.  Then the UFL objective corresponds exactly to the assignment problem for topic modeling.  Algorithm~\ref{alg:ufl} is a greedy algorithm for UFL that has been shown to achieve constant factor approximation guarantees when distances between clients and facilities forms a metric~\cite{greedy_fl} (this guarantee does not apply in our case, as KL-divergence is not a metric).

The algorithm, must select, among all topics and all unmarked tokens $T$, the minimizer to
\begin{equation}
\frac{f_i + \sum_{t \in T} \mbox{KL}(\tilde{w}_{jt},\psi_i)}{|T|}.
\label{eqn:fl}
\end{equation}
This algorithm appears to be computationally expensive, requiring multiple rounds of marking where each round requires us to find a minimizer over exponentially-sized sets.  Surprisingly, under mild assumptions we can use the structure of our problem to derive an efficient implementation of this algorithm that runs in total time $O(NK)$.  The details of this efficient implementation are presented in Appendix B.

\subsection{Incremental Topic Refinement}

Unlike traditional clustering problems, topic modeling is hierarchical: we have both word level assignments and ``mini-topics'' (formed by word tokens in the same document which are assigned to the same topic). Explicitly refining the mini-topics should help in achieving better word-coassignment within the same document. Inspired by local search techniques in the clustering literature \cite{local_search}, we take a similar approach here. However, traditional approaches \cite{local_search2} do not directly apply in our setting; we therefore adapt local search techniques from clustering to the topic modeling problem.
% Instead, we propose to adapt this idea in our setting, which turns out to be much more efficient and essential in learning good topics.
%In the k-means literature, methods that perform \textit{local search} in addition to standard batch updates of assignments and means have proven to be useful for a variety of problems~\cite{local_search}.  In the setting of k-means with KL-divergence in place of the squared Euclidean distance, local search appears to be particularly effective in avoiding local optima; in~\cite{local_search2}, it is shown how to efficiently compute document re-assignments in a local search fashion, and these updates are interleaved with standard batch updates to achieve significantly improved clustering results.

\begin{algorithm}[tb]
   \caption{Incremental Topic Refinements for $\mathbf{Z}$}
   \label{alg:localsearch}
\begin{algorithmic}
   \STATE {\bfseries Input:} Words: $\mathbf{W}$, Number of topics: $K$, Topic penalty: $\lambda$, Assignment: $\mathbf{Z}$, Topics: $\bm{\psi}$
   \STATE randomly permute the documents.
   \FOR{every document $j$}
   %\STATE Initialize the set of topics for $j$ to be empty: $F = \phi$. \{S | \exists i s.t. z_{jt} = i~\forall t \in S\}
   \FOR{each mini-topic $S$, where $z_{js} = i~\forall s \in S$ for some topic $i$}
   \FOR{every other topic $i' \neq i$}
   \STATE Compute $\Delta(S,i,i')$, the change in the obj. function when re-assigning $z_{js} = i'~\forall s \in S.$
   \ENDFOR
   \STATE Let $i^* = \mbox{argmin}_{i'} \Delta(S,i,i')$.
   \STATE Reassign tokens in $S$ to $i^*$ if it yields a smaller obj.
   \STATE Update topics $\bm{\psi}$ and assignments $\mathbf{Z}$.
   \ENDFOR 
   \ENDFOR
   \STATE {\bfseries Output:} Assignments $\mathbf{Z}$ and Topics $\bm{\psi}$.
\end{algorithmic}
\end{algorithm}

More specifically, we consider an incremental topic refinement scheme that works as follows.  For a given document, we consider swapping all word tokens assigned to the same topic within that document to another topic.  We compute the change in objective function that would occur if we both updated the topic assignments for those tokens and then updated the resulting topic vectors.  Specifically, for document $j$ and its mini-topic $S$ formed by its word tokens assigned to topic $i$, the objective function change can be computed by
\begin{displaymath}
	\Delta(S,i,i') = -(n_{\cdot\cdot}^{i}-n_{j\cdot}^{i})\phi(\bm{\psi}_i^{-}) - (n_{\cdot\cdot}^{i'}+n_{j\cdot}^{i})\phi(\bm{\psi}_{i'}^{+}) 
	+ n_{\cdot\cdot}^{i}\phi(\bm{\psi}_{i}) + n_{\cdot\cdot}^{i'}\phi(\bm{\psi}_{i'}) - \lambda\mathbb{I}[i'\in\mathcal{T}_j],
\end{displaymath}
where $n_{j\cdot}^{i}$ is the number of tokens in document $j$ assigned to topic $i$, $n_{\cdot\cdot}^{i}$ is the total number of tokens assigned to topic $i$, $\bm{\psi}_i^{-}$ and $\bm{\psi}_{i'}^{+}$ are the updated topics, $\mathcal{T}_j$ is the set of all the topics used in document $j$, and $\phi(\bm{\psi}_i) = \sum_{w}\psi_{iw}\log\psi_{iw}$.

We accept the move if $\min_{i'\neq i}\Delta(S,i,i')<0$ and update the topics $\bm{\psi}$ and assignments $\mathbf{Z}$ accordingly. Then we continue to the next mini-topic, hence the term ``incremental". Note here we accept \textit{all} moves that improve the objective function instead of just the single best move as in traditional approaches~\cite{local_search2}. Since $\bm{\psi}$ and $\mathbf{Z}$ are updated in every objective-decreasing move, we randomly permute the processing order of the documents in each iteration. This usually helps in obtaining better results in practice. See Algorithm~\ref{alg:localsearch} for details.
%The computation of $\Delta(S,i,i')$ can be performed in $O(|S|)$ time if we pre-compute all the necessary information. Since we compute the change across all topics, and across all sets $S$, the total running time of the incremental topic refinement can be seen to be $O(NK)$, as in the basic batch algorithm.

At first glance, it appears that this incremental topic refinement strategy may be computationally expensive.  However, computing the global change in objective function $\Delta(S,i,i')$ can be performed in $O(|S|)$ time, if the topics are maintained by count matrices. Only the counts involving the words in the mini-topic and the total counts are affected. Since we compute the change across all topics, and across all mini-topics $S$, the total running time of the incremental topic refinement can be seen to be $O(NK)$, as in the basic batch algorithm and the facility location assignment algorithm.
%The terms in the objective function that are affected by this swap are any words (in any document) assigned to either topic $i$ or topic $i'$.  Recall that the corresponding terms in the objective function are of the form $- \log \psi_{iu}$, where $\psi_{iu}$ equals the number of word tokens assigned to topic $i$, across all documents, divided by the number of occurrences of word $u$ in topic $i$, across all documents.  We know that the total number of word tokens increases by $|S|$ in topic $i'$ with the swap, and decreases by $|S|$ in topic $i$.  Similarly, we can compute the change in the number of occurrences of all tokens in $S$ in $|S|$ time, and use these values to update all the relevant terms of the objective function in $O(|S|)$ time.  Finally, since we compute the change across all topics, and across all sets $S$, the total running time of the incremental topic re-assignment can be seen to be $O(NK)$, as in the basic batch algorithm.

%%% Local Variables:
%%% mode: latex
%%% TeX-master: "main_nips"
%%% End:

\section{Experiments}

In this section, we compare the algorithms proposed above with their probabilistic counterparts.

\subsection{Synthetic Documents}

Our first set of experiments is on simulated data. We compare three versions of our algorithms---Basic Batch (\texttt{Basic}), Improved Word Assignment (\texttt{Word}), and Improved Word with Topic Refinement (\texttt{Word+Refine})---with the collapsed Gibbs sampler (\texttt{CGS})\footnote{http://psiexp.ss.uci.edu/research/programs\_data/toolbox.htm} \cite{lda_gibbs}, the standard variational inference algorithm (\texttt{VB})\footnote{http://scikit-learn.org/stable/modules/generated/sklearn.decomposition.LatentDirichletAllocation.html}~\cite{blei_lda}, and the recent \texttt{Anchor} method\footnote{http://www.cs.nyu.edu/$\sim$halpern/code.html} \cite{arora2012practical}.

\textbf{Methodology}. Due to a lack of ground truth data for topic modeling, following~\cite{arora2012practical}, we benchmark on synthetic data.  We train all algorithms on the following data sets. (A) documents sampled from an LDA model with $\alpha = 0.04, \beta = 0.05$, with 20 topics and having vocabulary size 2000. Each document has length 150. (B) documents sampled from an LDA model with $\alpha = 0.02, \beta = 0.01$, 50 topics and vocabulary size 3000. Each document has length 200.

\begin{figure}
\centering
\begin{minipage}{.6\textwidth}
\begin{tabular}{|l||c|c|c|c|c|c|}
\hline
\multirow{2}{*}{Method} & \multicolumn{6}{c|}{Number of Documents} \\ \cline{2-7}
 & 5k & 10k & 50k & 100k & 500k & 1M \\
\hline
\texttt{CGS} (s) & .143 & .321 & 1.96 & 4.31 & 23.36 & 55.69\\
\hline
\texttt{Word} (s) & .438 & .922 & 4.88 & 9.75 & 50.38 & 101.58\\
\hline
\texttt{Word}/\texttt{CGS} & 3.07 & 2.87 & 2.48 & 2.26 & 2.16 & 1.82\\
\hline
\texttt{Refine} (s) & .277 & .533 & 2.58 & 5.09 & 25.75 & 52.28\\
\hline
\texttt{Refine}/\texttt{CGS} & 1.94 & 1.66 & 1.32 & 1.18 & 1.10 & 0.94\\
\hline
\end{tabular}
\end{minipage}%
\begin{minipage}{.5\textwidth}
\begin{tabular}{| l || c c c |}
\hline
Algorithm & 6K & 8K & 10K \\ \hline
\texttt{CGS} &  \textbf{0.098} & 0.338 & 0.276 \\
\texttt{VB} &  0.448 & 0.443 & 0.392 \\
\texttt{Anchor} & 0.118 & 0.118 & 0.112 \\
\texttt{Basic} & 1.805 & 1.796 & 1.794 \\
\texttt{Word} &  0.582 & 0.537 & 0.504 \\
\texttt{W+R} & 0.155 & \textbf{0.110} & \textbf{0.105} \\
\texttt{KMeans} & 1.022 & 0.921 & 0.952 \\ \hline
\end{tabular}
\end{minipage}
\caption{\textbf{Left:} Running time comparison per iteration (in secs) of \texttt{CGS} to the facility location improved word algorithm (\texttt{Word}) and local refinement (\texttt{Refine}), on data sets of different sizes.  \texttt{Word}/\texttt{CGS}  and \texttt{Refine}/\texttt{CGS}  refer to the ratio of \texttt{Word} and \texttt{Refine} to \texttt{CGS}.  For larger datasets, \texttt{Word} takes roughly 2 Gibbs iterations and \texttt{Refine} takes roughly 1 Gibbs iteration. \textbf{Right:} Comparison of topic reconstruction errors of different algorithms with different sizes of SynthB.}
\label{fig:minipage}
\end{figure}

For the collapsed Gibbs sampler, we collect 10 samples with 30 iterations of thinning after 3000 burn-in iterations. The variational inference runs for 100 iterations.  The \texttt{Word} algorithm replaces basic word assignment with the improved word assignment step within the batch algorithm, and \texttt{Word+Refine} further alternates between improved word and incremental topic refinement steps. The \texttt{Word} and \texttt{Word+Refine} are run for 20 and 10 iterations respectively. For \texttt{Basic}, \texttt{Word} and \texttt{Word+Refine}, we run experiments with $\lambda\in\{6,7,8,9,10,11,12\}$, and the best results are presented if not stated otherwise. In contrast, the \textit{true} $\alpha, \beta$ parameters are provided as input to the LDA algorithms, whenever applicable.  We note that we have heavily handicapped our methods by this setup, since the LDA algorithms are designed specifically for data from the LDA model.

%\textbf{Objective optimization}. Table \ref{tab:obj} shows the optimized objective function values for all three proposed algorithms. We can see that the \texttt{Word} algorithm significantly reduces the objective value when compared with the \texttt{Basic} algorithm, and the \texttt{Word+Refine} algorithm reduces further. As pointed out in \cite{convex-dpmeans} in the context of other SVA models, the \texttt{Basic} algorithm is very sensitive to initializations. However, this is not the case for the \texttt{Word} and \texttt{Word+Refine} algorithms and they are quite robust to initializations. From the objective values, the improvement from \texttt{Word+Refine} to \texttt{Word} seems to be marginal. But we will show in the following that the incorporation of the topic refinement is crucial for learning good topic models.

%\begin{table}
%\centering
%\begin{tabular}{| l | c | c |}
%\hline
%objective value & SynthA & SynthB \\ \hline
%\texttt{Basic} & $5074939.616$ & $5453889.128$ \\ \hline
%\texttt{Word} & $4055091.759$ & $3790071.752$ \\ \hline
%\texttt{Word+Refine} & $3975536.098$ & $3609980.107$ \\ \hline
%\end{tabular}
%\caption{Optimized combinatorial topic modeling objective function values for different algorithms with $\lambda = 10$.}
%\label{tab:obj}
%\end{table}

\textbf{Assignment accuracy}. Both the Gibbs sampler and our algorithms provide word-level topic assignments. Thus we can compare the training accuracy of these assignments, which is shown in Table \ref{tab:nmi}.  The result of the Gibbs sampler is given by the highest among all the samples selected. The accuracy is shown in terms of the \textit{normalized mutual information} (NMI) score and the \textit{adjusted Rand index} (ARand), which are both in the range of [0,1] and are standard evaluation metrics for clustering problems.
%The \texttt{Basic} algorithm performs poorly despite its similarity with the Gibbs sampler.
%Unlike \texttt{Basic}, the \texttt{Word} algorithm greatly boosts the assignment accuracy, which shows that the algorithm has already been doing something reasonable. With further help from topic refinement, we match or marginally exceed the performance of the Gibbs sampler.
From the plots, we can see that the performance of \texttt{Word+Refine} matches or slightly outperforms the Gibbs sampler for a wide range of $\lambda$ values.

\textbf{Topic reconstruction error}. Now we look at the reconstruction error between the true topic-word distributions and the learned distributions. In particular, given a learned topic matrix $\hat{\bm{\psi}}$ and the true matrix $\bm{\psi}$, we use the Hungarian algorithm \cite{Hungarian} to align topics, and then evaluate the $\ell_1$ distance between each pair of topics. Figure \ref{fig:minipage} presents the mean reconstruction errors per topic of different learning algorithms for varying number of documents. As a baseline, we also include the results from the $k$-means algorithm with KL-divergence \cite{banerjee_bregman} where each document is assigned to a single topic.  We see that, on this data, the \texttt{Anchor} and \texttt{Word+Refine} methods perform the best; see Appendix C for further results and discussion.

\textbf{Running Time.}  See Figure~\ref{fig:minipage} for comparisons of our approach to \texttt{CGS}. The two most expensive steps of the \texttt{Word+Refine} algorithm are the word assignments via facility location and the local refinement step (the other steps of the algorithm are lower-order).  The relative runnings times improve as the data set sizes gets larger and, on large data sets, an iteration of \texttt{Refine} is roughly equivalent to one Gibbs iteration while an iteration of \texttt{Word} is roughly equivalent to two Gibbs iterations.  Since one typically runs thousands of Gibbs iterations (while ours runs in 10 iterations even on very large data sets), we can observe several orders of magnitude improvement in speed by our algorithm.  Further, running times could be significantly enhanced by noting that the facility location algorithm trivially parallellizes.  In addition to these results, we found our per-iteration running times to be consistently faster than \texttt{VB}.

See Appendix C for further results on synthetic data, including on using our algorithm as initialization to the collapsed Gibbs sampler.

\begin{table}
\centering
\begin{tabular}{| l || c c c c c |}
\hline
NMI / ARand & $\lambda = 8$ & $\lambda = 9$ & $\lambda = 10$ & $\lambda = 11$ & $\lambda = 12$ \\ \hline
\texttt{Basic} & 0.027 / 0.009 & 0.027 / 0.009 & 0.027 / 0.009 & 0.027 / 0.009 & 0.027 / 0.009 \\ \hline
\texttt{Word} & 0.724 / 0.669 & 0.730 / 0.660 & 0.786 / 0.750 & 0.786 / 0.745 & 0.784 / 0.737 \\ \hline
\texttt{Word+Refine} & 0.828 / 0.838 & 0.839 / 0.850 & 0.825 / 0.810 & 0.847/ \textbf{0.859} & \textbf{0.848} / \textbf{0.859} \\ \hline
\texttt{CGS} & \multicolumn{5}{c|}{0.829 / 0.839} \\ \hline \hline

NMI / ARand & $\lambda = 6$ & $\lambda = 7$ & $\lambda = 8$ & $\lambda = 9$ & $\lambda = 10$ \\ \hline
\texttt{Basic} & 0.043 / 0.007 & 0.043 / 0.007 & 0.043 / 0.007 & 0.043 / 0.007 & 0.043 / 0.007 \\ \hline
\texttt{Word} & 0.850 / 0.737 & 0.854 / 0.743 & 0.855 / 0.752 & 0.855 / 0.750 & 0.850 / 0.743 \\ \hline
\texttt{Word+Refine} & 0.922 / 0.886 & \textbf{0.926} / \textbf{0.901} & 0.913 / 0.860 & 0.923 / 0.899 & 0.914 / 0.876 \\ \hline
\texttt{CGS} & \multicolumn{5}{c|}{0.917 / 0.873} \\ \hline
\end{tabular}
\caption{The NMI scores and Adjusted Rand Index (best results in bold) for word assignments of our algorithms for both synthetic datasets with 5000 documents (\textbf{top}: SynthA, \textbf{bottom}: SynthB).}
\label{tab:nmi}
\end{table}

\begin{table}
\centering
\begin{tabular}{| l || c | c | c || c | c | c || c | c | c |}
\hline
\multirow{2}{*}{Enron} & \multicolumn{3}{c||}{$\beta=0.1$}  & \multicolumn{3}{c||}{$\beta=0.01$} & \multicolumn{3}{c|}{$\beta=0.001$} \\ \cline{2-10}
 & hard & original & KL & hard & original & KL & hard & original & KL \\ \hline
 \texttt{CGS} & -5.932 & -8.583 & 3.899 & -5.484 & -10.781 & 7.084 & -5.091 & -13.296 & 10.000 \\ \hline
 \texttt{W+R} & -5.434 & -9.843 & 4.541 & -5.147 & -11.673 & 7.225 & -4.918 & -13.737 & 9.769 \\ \hline
 \multirow{2}{*}{NYT} & \multicolumn{3}{c||}{$\beta=0.1$}  & \multicolumn{3}{c||}{$\beta=0.01$} & \multicolumn{3}{c|}{$\beta=0.001$} \\ \cline{2-10}
 & hard & original & KL & hard & original & KL & hard & original & KL \\ \hline
 \texttt{CGS} & -6.594 & -9.361 & 4.374 & -6.205 & -11.381 & 7.379 & -5.891 & -13.716 & 10.135 \\ \hline
 \texttt{W+R} & -6.105 & -10.612 & 5.059 & -5.941 & -12.225 & 7.315 & -5.633 & -14.524 & 9.939 \\ \hline
\end{tabular}
\caption{The predictive word log-likelihood on new documents for Enron ($K=100$ topics) and NYTimes ($K=100$ topics) datasets with fixed $\alpha$ value. ``hard'' is short for hard predictive word log-likelihood which is computed using the word-topic assignments inferred by the \texttt{Word} algorithm, ``original'' is short for original predictive word log-likelihood which is computed using the document-topic distributions inferred by the sampler, and ``KL'' is short for symmetric KL-divergence.}
\label{tab:enron}
\end{table}

\begin{table}[!h]\small
\centering
\begin{tabular}{| l || l |}
\hline
\texttt{CGS} & art, artist, painting, museum, century, show, collection, history, french, exhibition\\
\texttt{W+R} & painting, exhibition, portrait, drawing, object, photograph, gallery, flag, artist \\ \hline
\texttt{CGS} & plane, flight, airport, passenger, pilot, aircraft, crew, planes, air, jet\\
\texttt{W+R} & flight, plane, passenger, airport, pilot, airline, aircraft, jet, planes, airlines\\ \hline
\texttt{CGS} & money, million, fund, donation, pay, dollar, contribution, donor, raising, financial\\
\texttt{W+R} & fund, raising, contribution, donation, raised, donor, soft, raise, finance, foundation\\ \hline
\texttt{CGS} & car, driver, truck, vehicles, vehicle, zzz\_ford, seat, wheel, driving, drive\\
\texttt{W+R} & car, driver, vehicles, vehicle, truck, wheel, fuel, engine, drive, zzz\_ford\\ \hline
\end{tabular}
\caption{Example topics pairs learned from NYTimes dataset.}
\label{tab:enron-topics}
\end{table}

\subsection{Real Documents}

We consider two real-world data sets with different properties: a random subset of the Enron emails (8K documents, vocabulary size 5000), and a subset of the New York Times articles\footnote{http://archive.ics.uci.edu/ml/machine-learning-databases/bag-of-words/} (15K documents, vocabulary size 7000). 1K documents are reserved for predictive performance assessment for both datasets. We use the following metrics: a ``hard'' predictive word log-likelihood and the standard probabilistic predictive word log-likelihood on new documents. To get the topic assignments for new documents, we can either perform one iteration of the \texttt{Word} algorithm which can be used to compute the ``hard'' predictive log-likelihood, or use MCMC to sample the assignments with the learned topic matrix. Our hard log-likelihood can be viewed as the natural analogue of the standard predictive log-likelihood to our setting.  We also compute the symmetric KL-divergence between learned topics. To make fair comparisons, we tune the $\lambda$ value such that the resulting number of topics per document is comparable to that of the sampler.  We remind the reader of issues raised in the introduction, namely that our combinatorial approach is no longer probabilistic, and therefore would not necessarily be expected to perform well on a standard likelihood-based score.

%Figure \ref{fig:real} shows the evolution of the log-likelihood of the trained models. Similar to what we have observed earlier, the proposed algorithm can quickly reach a state where the sampler needs thousands of passes over the data set. Again, we can use the semi-optimized assignments as initializations to the sampler in order to speed-up convergence, which we will leave for future research.

Table \ref{tab:enron} shows the results on the Enron and NYTimes datasets. We can see that our approach excels in the ``hard'' predictive word log-likelihood while lags in the standard mixture-view predictive word log-likelihood, which is in line with the objectives and reminiscent to the differences between $k$-means and GMMs. Table \ref{tab:enron-topics} further shows some sample topics generated by \texttt{CGS} and our method. See Appendix C for further results on predictive log-likelihood, including comparisons to other approaches than \texttt{CGS}.
%In our approach, only a single topic is responsible for the generation of each word, which leans towards to more orthogonal topic matrix. This can also be seen from Table \ref{tab:enron}, where the difference of the predictive log-likelihood between \texttt{CGS} and \texttt{Word+Refine} decreases as $\beta$ decreases (favoring orthogonal topics). 
%In Figure \ref{fig:trend}, we show the evolution of original predictive word log-likelihood initialized with the \texttt{Word+Refine} optimized assignment for 3 iterations. With these semi-optimized initializations, we observe quite a speed-up over the random initialization. 

%Table \ref{tab:enron-topics} further shows some sample topics generated by \texttt{CGS} and our method (the full list is shown in the supplementary material).

\section{Conclusions}

Our goal has been to lay the groundwork for a combinatorial optimization view of topic modeling as an alternative to the standard probabilistic framework.  Small-variance asymptotics provides a natural way to obtain an underlying objective function, using the $k$-means connection to Gaussian mixtures as an analogy.  Potential future work includes distributed implementations for further scalability, adapting k-means-based semi-supervised clustering techniques to this setting, and extensions of k-means++~\cite{kmeans_plus_plus} to derive explicit performance bounds for this problem.

\bibliographystyle{plain}
\bibliography{ldabib}

\begin{appendices}
\section{Full Derivation of the SVA Objective}
Recall the standard Latent Dirichlet Allocation (LDA) model:
\begin{itemize}
\item Choose $\theta_j \sim \mbox{Dir}(\alpha)$, where $j \in \{1, ..., M\}$.
\item Choose $\psi_i \sim \mbox{Dir}(\beta)$, where $i \in \{1, ..., K\}$.
\item For each word $t$ in document $j$:
\begin{itemize}
\item Choose a topic $z_{jt} \sim \mbox{Cat}(\theta_j)$.
\item Choose a word $w_{jt} \sim \mbox{Cat}(\psi_{z_{jt}})$.
\end{itemize}
\end{itemize}
Here $\alpha$ and $\beta$ are scalar-valued (i.e., we are using a symmetric Dirichlet distribution).  Denote $\mathbf{W}$ as the vector denoting all words in all documents, $\mathbf{Z}$ as the topic indicators of all words in all documents, $\bm{\theta}$ as the concatenation of all the $\theta_j$ variables, and $\bm{\psi}$ as the concatenation of all the $\psi_i$ variables.  Also let $N_j$ be the total number of word tokens in document $j$.  The $\theta_j$ vectors are each of length $K$, the number of topics.  The $\psi_i$ vectors are each of length $D$, the number of words in the dictionary.   We can write down the full joint likelihood of the model as $p(\mathbf{W},\mathbf{Z},\bm{\theta},\bm{\psi} | \alpha, \beta) =$
\begin{displaymath}
\prod_{i=1}^K p(\psi_i | \beta) \prod_{j=1}^M p(\theta_j | \alpha) \prod_{t=1}^{N_j} p(z_{jt} | \theta_j)p (w_{jt} | \psi_{z_{jt}}),
\end{displaymath}
where each of the probabilities are given as specified in the above model.  Now, following standard LDA manipulations, we can eliminate variables to simplify inference by integrating out $\bm{\theta}$ to obtain
\begin{displaymath}
p(\mathbf{Z},\mathbf{W}, \bm{\psi} | \alpha, \beta) = \int_{\bm{\theta}} p(\mathbf{W},\mathbf{Z},\bm{\theta},\bm{\psi} | \alpha, \beta) d \bm{\theta}.
\end{displaymath}
%When integrating out a Dirichlet distribution times a multinomial distribution, i.e. $\int p(z| \theta) p(\theta | \alpha) d \theta$, where $p(z | \theta)$ is multinomial and $p(\theta | \alpha)$ is Dirichlet, one obtains the \textit{Dirichlet-multinomial} distribution.  We therefore obtain $p(\bm{Z},\bm{W}, \bm{\psi} | \alpha, \beta) = $
After simplification, we obtain $p(\mathbf{Z},\mathbf{W}, \bm{\psi} | \alpha, \beta) =$
\begin{displaymath}
%&  & \prod_{i=1}^K p(\psi_i | \beta) \prod_{j=1}^M \prod_{t=1}^{N_j} p(w_{jt} | \psi_{z_{jt}}) \int_{\bm{\theta}} \prod_{j=1}^M p(\theta_j | \alpha) \prod_{t=1}^{N_j} p(z_{jt} | \theta_j) d \bm{\theta}\\
\bigg [\prod_{i=1}^K p(\psi_i | \beta) \prod_{j=1}^M \prod_{t=1}^{N_j} p(w_{jt} | \psi_{z_{jt}}) \bigg ] \times 
\bigg [ \prod_{j=1}^M \frac{\Gamma(\alpha K)}{\Gamma(\sum_{i=1}^K n_{j \cdot}^i + \alpha K)} \prod_{i=1}^K \frac{\Gamma(n_{j \cdot}^i + \alpha)}{\Gamma(\alpha)} \bigg ].
\end{displaymath}
Here $n_{j \cdot}^i$ is the number of word tokens in document $j$ assigned to topic $i$. Now, following~\cite{mad_bayes}, we can obtain the SVA objective by taking the log of this likelihood and observing what happens when the variance goes to zero.  In order to do this, we must be able to scale the likelihood categorical distribution, which is not readily apparent.  Here we use two facts about the categorical distribution.  First, as discussed in~\cite{banerjee_bregman}, we can equivalently express the distribution $p(w_{jt} | \psi_{z_{jt}})$ in its Bregman divergence form, which will prove amenable to SVA analysis.  In particular, example 10 from~\cite{banerjee_bregman} details this derivation.  In our case we have a categorical distribution, and thus we can write the probability of token $w_{jt}$ as:
\begin{equation}
p(w_{jt} | \psi_{z_{jt}}) = \exp(-d_{\phi}(1,\psi_{z_{jt},w_{jt}})).
\label{eqn:bregkl}
\end{equation}
$d_{\phi}$ is the unique Bregman divergence associated with the categorical distribution which, as detailed in example 10 from~\cite{banerjee_bregman}, is the discrete KL divergence 
and $\psi_{z_{jt},w_{jt}}$ is the entry of the topic  vector associated with the topic indexed by $z_{jt}$ at the entry corresponding to the word at token $w_{jt}$.  
This KL divergence will correspond to a single term of the form $x \log(x / y)$, where $x = 1$ since we are considering a single token of a word in a document.  Thus, for a particular token, the KL divergence simply equals $-\log \psi_{z_{jt},w_{jt}}$.  Note that when plugging in $-\log \psi_{z_{jt},w_{jt}}$ into \eqref{eqn:bregkl}, we obtain exactly the original probability for word token $w_{jt}$ that we had in the original multinomial distribution.  We will write the KL-divergence $d_{\phi}(1,\psi_{z_{jt},w_{jt}})$ as $\mbox{KL}(\tilde{w}_{jt},\psi_{z_{jt}})$, where $\tilde{w}_{jt}$ is an indicator vector for the word at token $w_{jt}$.

Although it may appear that we have gained nothing by this notational manipulation, there is a key advantage of expressing the categorical probability in terms of Bregman divergences. In particular, the second step is to parameterize the Bregman divergence by an additional variance parameter.  As discussed in Lemma 3.1 of~\cite{jiang_nips}, we can introduce another parameter, which we will call $\eta$, that scales the variance in an exponential family while fixing the mean.  This new distribution may be represented, using the Bregman divergence view, as proportional to $\exp(- \eta \cdot \mbox{KL}(\tilde{w}_{jt},\psi_{z_{jt}}))$.  As $\eta \to \infty$, the mean remains fixed while the variance goes to zero, which is precisely what we require to perform small-variance analysis.

We will choose to scale $\alpha$ appropriately as well; this will ensure that the hierarchical form of the model is retained asymptotically.  In particular, we will write $\alpha = \exp(-\lambda \cdot \eta)$.  Now we consider the full negative log-likelihood:
\begin{displaymath}
- \log p(\mathbf{Z},\mathbf{W}, \bm{\psi} | \alpha, \beta).
\end{displaymath}
Let us first derive the asymptotic behavior arising from the Dirichlet-multinomial distribution part of the likelihood, for a given document $j$:
\begin{displaymath}
\frac{\Gamma(\alpha K)}{\Gamma(\sum_{i=1}^K n_{j \cdot}^i + \alpha K)} \prod_{i=1}^K \frac{\Gamma(n_{j \cdot}^i + \alpha)}{\Gamma(\alpha)}.
\end{displaymath}
In particular, we will show the following lemma.
\begin{lem}
Consider the likelihood
\begin{displaymath}
p(\bm{Z} | \alpha) = \bigg [ \prod_{j=1}^M \frac{\Gamma(\alpha K)}{\Gamma(\sum_{i=1}^K n_{j \cdot}^i + \alpha K)} \prod_{i=1}^K \frac{\Gamma(n_{j \cdot}^i + \alpha)}{\Gamma(\alpha)} \bigg ].
\end{displaymath}
If $\alpha = \exp(-\lambda \cdot \eta)$, then asymptotically as $\eta \to \infty$ we have
\begin{displaymath}
-\log p(\bm{Z} | \alpha) \sim \eta \lambda \sum\nolimits_{j=1}^M(K_{j+} - 1).
\end{displaymath}
\end{lem}
\begin{proof}
Note that $N_j = \sum_{i=1}^K n^i_{j \cdot}$.  Using standard properties of the $\Gamma$ function, we have that the negative log of the above distribution is equal to 
\begin{displaymath}
\sum_{n=0}^{N_j - 1} \log (\alpha K + n) - \sum_{i=1}^K \sum_{n=0}^{n_{j \cdot}^i - 1} \log(\alpha + n).
\end{displaymath}
All of the logarithmic summands converge to a finite constant whenever they have an additional term besides $\alpha$ or $\alpha K$ inside.  The only terms that asymptotically diverge are those of the form $\log (\alpha K)$ or $\log (\alpha)$, that is, when $n = 0$.  The first term always occurs.  Terms of the type $\log(\alpha)$ occur only when, for the corresponding $i$, we have $n_{j \cdot}^i > 0$.  Recalling that $\alpha = \exp(-\lambda \cdot \eta)$, we can conclude that the negative log of the Dirichlet multinomial term becomes asymptotically $ \eta \lambda (K_{j+} - 1)$, where $K_{j+}$ is the number of topics $i$ in document $j$ where $n_{j \cdot}^i > 0$, i.e., the number of topics currently utilized by document $j$.  (The maximum value for $K_{j+}$ is $K$, the total number of topics.)
\end{proof}
The rest of the negative log-likelihood is straightforward.  The $-\log p(\psi_i | \beta)$ terms vanish asymptotically since we are not scaling $\beta$ (see the note below on scaling $\beta$).  Thus, the remaining terms in the SVA objective are the ones arising from the word likelihoods which, after applying a negative logarithm, become
\begin{displaymath}
-\sum_{j=1}^M \sum_{t=1}^{N_j} \log p(w_{jt} | \psi_{z_{jt}}).
\end{displaymath}
Using the Bregman divergence representation, we can conclude that the negative log-likelihood asymptotically yields the objective $-\log p(\mathbf{Z},\mathbf{W}, \bm{\psi} | \alpha, \beta) \sim$
\begin{displaymath}
\eta \bigg [ \sum_{j=1}^M \sum_{t=1}^{N_j} \mbox{KL}(\tilde{w}_{jt},\psi_{z_{jt}}) + \lambda \sum_{j=1}^M (K_{j+}-1) \bigg ],
\end{displaymath}
where $f(x) \sim g(x)$ denotes that $f(x) / g(x) \to 1$ as $x \to \infty$.  This leads to the objective function
\begin{equation}
\min_{\bm{Z},\bm{\psi}} \sum_{j=1}^M \sum_{t=1}^{N_j} \mbox{KL}(\tilde{w}_{jt},\psi_{z_{jt}}) + \lambda \sum_{j=1}^M K_{j+}.
\label{eqn:obj}
\end{equation}
We remind the reader that $\mbox{KL}(\tilde{w}_{jt},\psi_{z_{jt}}) = -\log \psi_{z_{jt},w_{jt}}$.  Thus we obtain a $k$-means-like term that says that all words in all documents should be ``close" to their assigned topic in terms of KL-divergence, but that we should also not have too many topics represented in each document.

Note that we did not scale $\beta$, to obtain a simple objective with only one parameter (other than the total number of topics), but let us say a few words about scaling $\beta$.  A natural approach is to further integrate out $\bm{\psi}$ of the joint likelihood, as is done with the collapsed Gibbs sampler.  One would obtain additional Dirichlet-multinomial distributions, and properly scaling as discussed above would yield a simple objective that places penalties on the number of topics per document as well as the number of words in each topic.  Optimization would be performed only with respect to the topic assignment matrix.  Future work would consider the effectiveness of such an objective function for topic modeling.  
%I have not worked out what happens if we try to scale $\beta$ on the above likelihood, but another option is to integrate out $\bm{\psi}$, as is done with the collapsed Gibbs sampler for LDA.  This would yield an analogous Dirichlet-multinomial distribution; in detail, we would obtain
%\begin{eqnarray*}
%p(\bm{Z},\bm{W} | \alpha, \beta) & = & \int_{\bm{\psi}} \int_{\bm{\theta}} p(\bm{W},\bm{Z},\bm{\theta},\bm{\psi} | \alpha, \beta) d \bm{\psi} d \bm{\theta}\\
%& = &\bigg [ \prod_{i=1}^K \frac{\Gamma(\beta V)}{\Gamma (\sum_{r=1}^V n_{\cdot r}^i + \beta V)} \prod_{r=1}^V \frac{\Gamma(n_{\cdot r}^i + \beta)}{\Gamma(\beta)} \bigg ] \bigg [ \prod_{j=1}^M \frac{\Gamma(\alpha K)}{\Gamma(\sum_{i=1}^K n_{j \cdot}^i + \alpha K)} \prod_{i=1}^K \frac{\Gamma(n_{j \cdot}^i + \alpha)}{\Gamma(\alpha)} \bigg ],
%\end{eqnarray*}
%where $V$ is the size of the vocabulary and $n_{\cdot r}^i$ is the number of instances of word $r$ in topic $i$ across all documents.  Performing asymptotics with a similar scaling of $\beta$ as we had for $\alpha$, we would obtain a penalty term that accounts for the number of words assigned to each topic.  In other words, we would be left only with two penalties, one on the number of topics per document and the other on the number of words per topic.  (NB: At one point I thought about this objective, and if I recall correctly I concluded that it has a trivial solution, and is thus not interesting.)

\subsection{Further Details on the Basic Algorithm}
\begin{algorithm}[tb]
  \caption{Basic Batch Algorithm}
  \label{alg:basic}
\begin{algorithmic}
   \STATE {\bfseries Input:} Words: $\mathbf{W}$, Number of topics: $K$, Topic penalty: $\lambda$
   \STATE Initialize $\mathbf{Z}$ and topic vectors $\psi_1, ..., \psi_K$.
   \STATE Compute initial objective function~\eqref{eqn:obj} using $\mathbf{Z}$ and $\bm{\psi}$.
   \REPEAT
   \STATE \textit{//Update assignments:}
   \FOR{every word token $t$ in every document $j$}
   \STATE Compute distance $d(j,t,i)$ to topic $i$: $-\log(\psi_{i,w_{jt}})$.
   \STATE If $z_{jt} \neq i$ for all tokens $t$ in document $j$, add $\lambda$ to $d(j,t,i)$.
   \STATE Obtain assignments via $Z_{jt} = \mbox{argmin}_i d(j,t,i)$.
   \ENDFOR
   \STATE \textit{//Update topic vectors:}
   \FOR{every element $\psi_{iu}$}
   \STATE $\psi_{iu} = $ \# occ. of word $u$ in topic $i$ / total \# of word tokens in topic $i$.
   \ENDFOR
   \STATE Recompute objective function~\eqref{eqn:obj} using updated $\mathbf{Z}$ and $\bm{\psi}$.
   \UNTIL{no change in objective function.}
   \STATE {\bfseries Output:} Assignments $\mathbf{Z}$.
\end{algorithmic}
\end{algorithm}
Pseudo-code for the basic algorithm is given as Algorithm~\ref{alg:basic}.  We briefly elaborate on a few points raised in the main text.

First, the running time of the batch algorithm can be shown to be $O(NK)$ per iteration, where $N$ is the total number of word tokens and $K$ is the number of topics.  
This is because each word token must be compared to every topic, but the resulting comparison can be done in constant time.  Updating topics is performed by maintaining a count of the number of occurrences of each word in each topic, which also runs in $O(NK)$ time.  Note that the collapsed Gibbs sampler runs in $O(NK)$ time per iteration, and thus has a comparable running time per iteration.

Second, one can also show that this algorithm is guaranteed to converge to a local optimum, similar to $k$-means and DP-means.  
The argument follows along similar lines to $k$-means and DP-means, namely that each updating step cannot increase the objective function.  In particular, the update on the topic vectors must improve the objective function since the means are known to be the best representatives for topics based on the results of~\cite{banerjee_bregman}.  The assignment step must decrease the objective since we only re-assign if the distance goes down.  Further, we only re-assign to a topic that is not currently used by the document if the distance is more than $\lambda$ greater than the distance to the current topic, thus accounting for the additional $\lambda$ that must be paid in the objective function.

\iffalse
\begin{algorithm}[tb]
   \caption{Improved Word Assignments for $\mathbf{Z}$}
   \label{alg:ufl}
\begin{algorithmic}
   \STATE {\bfseries Input:} Words: $\mathbf{W}$, Number of topics: $K$, Topic penalty: $\lambda$, Topics: $\bm{\psi}$
   \FOR{every document $j$}
   %\STATE Initialize the set of topics for $j$ to be empty: $F = \phi$.
   \STATE Let $f_i = \lambda$ for all topics $i$.
   \STATE Initialize all word tokens to be unmarked.
   \WHILE{there are unmarked tokens}
   \STATE Pick the topic $i$ and set of unmarked tokens $T$ that minimizes
   \begin{equation}
   \frac{f_i + \sum_{t \in T} \mbox{KL}(\tilde{w}_{jt},\psi_i)}{|T|}.
   \label{eqn:scoring}
   \end{equation}
   \STATE Let $f_i = 0$ and mark all tokens in $T$.
   \STATE Assign $z_{jt} = i$ for all $t \in T$.
   \ENDWHILE
   \ENDFOR
   \STATE {\bfseries Output:} Assignments $\mathbf{Z}$.
\end{algorithmic}
\end{algorithm}
\fi

\section{An Efficient Facility Location Algorithm for Improved Word Assignments}
In this section, we describe an efficient $O(NK)$ algorithm based on facility location for obtaining the word assignments.  Recall the algorithm, given earlier in Algorithm~\ref{alg:ufl}.  Our first observation is that, for a fixed size of $T$ and a given $i$, the best choice of $T$ is obtained by selecting the $|T|$ closest tokens to $\psi_i$ in terms of the KL-divergence.  Thus, as a first pass, we can obtain the correct points to mark by appropriately sorting KL-divergences of all tokens to all topics, and then searching over all sizes of $T$ and topics $i$.  

Next we make three observations about the sorting procedure.  First, the KL-divergence between a word and a topic depends purely on counts of words within topics; recall that it is of the form $-\log \psi_{iu}$, where $\psi_{iu}$ equals the number of occurrences of word $u$ in topic $i$ divided by the total number of word tokens assigned to $i$.  Thus, for a given topic, the sorted words are obtained exactly by sorting word counts within a topic in decreasing order.  

Second, because the word counts are all integers, we can use a linear-time sorting algorithm such as counting sort or radix sort to efficiently sort the items.  In the case of counting sort, for instance, if we have $n$ integers whose maximum value is $k$, the total running time is $O(n+k)$; the storage time is also $O(n+k)$.  In our case, we perform many sorts.  Each sort considers, for a fixed document $j$, sorting word counts to some topic $i$.  Suppose there are $n^i_{j \cdot}$ tokens with non-zero counts to the topic, with maximum word count $m^i_j$.  Then the running time of this sort is $O(n^i_{j \cdot} + m^i_j)$.  
Across the document, we do this for every topic, making the running time scale as $O(\sum_i (n^i_{j \cdot} + m^i_j)) = O(n^{\cdot}_{j \cdot} K)$, where $n^{\cdot}_{j \cdot}$ is the number of word tokens in document $j$.  Across all documents this sorting then takes $O(NK)$ time.

Third, we note that we need only sort once per run of the algorithm.  Once we have sorted lists for words to topics, if we mark some set $T$, we can efficiently remove these words from the sorted lists and keep the updated lists in sorted order.  Removing an individual word from a single sorted list can be done in constant time by maintaining appropriate pointers.  Since each word token is removed exactly once during the algorithm, and must be removed from each topic, the total time to update the sorted lists during the algorithm is $O(NK)$.

At this point, we still do not have a procedure that runs in $O(NK)$ time.  In particular, we must find the minimum of
\begin{displaymath}
\frac{f_i + \sum_{t \in T} \mbox{KL}(\tilde{w}_{jt},\psi_i)}{|T|}
\end{displaymath}
at each round of marking.  Naively this is performed by traversing the sorted lists and accumulating the value of the above score via summation.  In the worst case, each round would take a total of $O(NK)$ time across all documents, so if there are $R$ rounds on average across all the documents, the total running time would be $O(NKR)$.  However, we can observe that we need not traverse entire sorted lists in general.  Consider a fixed document, where we try to find the best set $T$ by traversing all possible sizes of $T$.  We can show that, as we increase the size of $T$, the value of the score function monotonically decreases until hitting the minimum value, and then monotonically increases afterward.  
We can formalize the monotonicity of the scoring function as follows:
\begin{prop}
Let $s_{ni}$ be the value of the scoring function~\eqref{eqn:fl} for the best candidate set $T$ of size $t$ for topic $i$.  If $s_{t-1,i} \leq s_{ti}$, then $s_{ti} \leq s_{t+1,i}$.
\end{prop}
\begin{proof}
Recall that the KL-divergence is equal to the negative logarithm of the number of occurrences of the corresponding word token divided by the total number occurrences of tokens in the topic.  Write this as $\log n^i_{\cdot \cdot}- \log c_{i \ell}$, where $n^i_{\cdot \cdot}$ is the number of occurrences of tokens in topic $i$ and $c_{i \ell}$ is the count of the $\ell$-th highest-count word in topic $i$.  Now, by assumption $s_{t-1,i} \leq s_{ti}$.  Plugging the score functions into this inequality and cancelling the $\log n^i_{\cdot \cdot}$ terms, we have
\begin{displaymath}
-\frac{1}{t-1} \sum_{\ell=1}^{t-1} \log c_{i\ell} + \frac{f_i}{t-1} \leq -\frac{1}{t} \sum_{\ell=1}^t \log c_{i\ell} + \frac{f_i}{t}.
\end{displaymath}
Multiplying by $t(t-1)$ and simplifying yields the inequality
\begin{displaymath}
f_i + t \log c_{it} \leq \sum_{\ell=1}^{t} \log c_{i\ell}.
\end{displaymath}
Now, assuming this holds for $s_{t-1,i}$ and $s_{t,i}$, we must show that this inequality also holds for $s_{t,i}$ and $s_{t+1,i}$, i.e. that
\begin{displaymath}
f_i + (t+1) \log c_{i,t+1} \leq  \sum_{\ell=1}^{t+1} \log c_{i\ell}.
\end{displaymath}
Simple algebraic manipulation and the fact that the counts are sorted, i.e., $\log c_{i,t+1} \leq \log c_{it}$, shows the inequality to hold.
\end{proof}
In words, the above proof demonstrates that, once the scoring function stops decreasing, it will not decrease any further, i.e., the minimum score has been found.  Thus, once the score function starts to increase as $T$ gets larger, we can stop and the best score (i.e., the best set $T$) for that topic $i$ has been found.  We do this for all topics $i$ until we find the overall best set $T$ for marking.  Under the mild assumption that the size of the chosen minimizer $T$ is similar (specifically, within a constant factor) to the average size of the best candidate sets $T$ across the other topics (an assumption which holds in practice), then it follows that the total time to find all the sets $T$ takes $O(NK)$ time.

Putting everything together, all the steps of this algorithm combine to cost $O(NK)$ time.

%Note that sorting need only be performed once; the total cost of sorting distances between tokens to topics across all documents can be shown to be $O(NK)$ since sorting distances is equivalent to sorting the number of occurrences of words in topics, and can therefore be performed in linear time in the number of word tokens, using counting or radix sort.

\section{Additional Experimental Results}
\textbf{Objective optimization}. Table \ref{tab:obj} shows the optimized objective function values for all three proposed algorithms. We can see that the \texttt{Word} algorithm significantly reduces the objective value when compared with the \texttt{Basic} algorithm, and the \texttt{Word+Refine} algorithm reduces further. As pointed out in \cite{convex-dpmeans} in the context of other SVA models, the \texttt{Basic} algorithm is very sensitive to initializations. However, this is not the case for the \texttt{Word} and \texttt{Word+Refine} algorithms and they are quite robust to initializations. From the objective values, the improvement from \texttt{Word+Refine} to \texttt{Word} seems to be marginal. But we will show in the following that the incorporation of the topic refinement is crucial for learning good topic models.

\textbf{Evolution of the Gibbs Sampler.} The Gibbs sampler can easily become trapped in a local optima area and needs many iterations on large data sets, which can be seen from Figure \ref{fig:l1-trend}. Since our algorithm outputs $\mathbf{Z}$, we can use this assignment as initialization to the sampler. In Figure \ref{fig:l1-trend}, we also show the evolution of topic reconstruction $\ell_1$ error initialized with the \texttt{Word+Refine} optimized assignment for 3 iterations with varying values of $\lambda$. With these semi-optimized initializations, we observe more than 5-fold speed-up compared to random initializations. %We hypothesize that the chain will converge much faster if the \texttt{Word+Refine} optimization step is applied whenever the chain stays in a plateau.

\begin{table}
\centering
\begin{tabular}{| l | c | c |}
\hline
objective value ($\times10^6$) & SynthA & SynthB \\ \hline
%\texttt{Basic} & $5074939.616$ & $5453889.128$ \\ \hline
%\texttt{Word} & $4055091.759$ & $3790071.752$ \\ \hline
%\texttt{Word+Refine} & $3975536.098$ & $3609980.107$ \\ \hline
\texttt{Basic} & 5.07 & 5.45 \\ \hline
\texttt{Word} & 4.06 & 3.79 \\ \hline
\texttt{Word+Refine} & 3.98 & 3.61 \\ \hline
\end{tabular}
\caption{Optimized combinatorial topic modeling objective function values for different algorithms with $\lambda = 10$.}
\label{tab:obj}
\end{table}

\begin{figure}
\centering
\includegraphics[width=.8\textwidth]{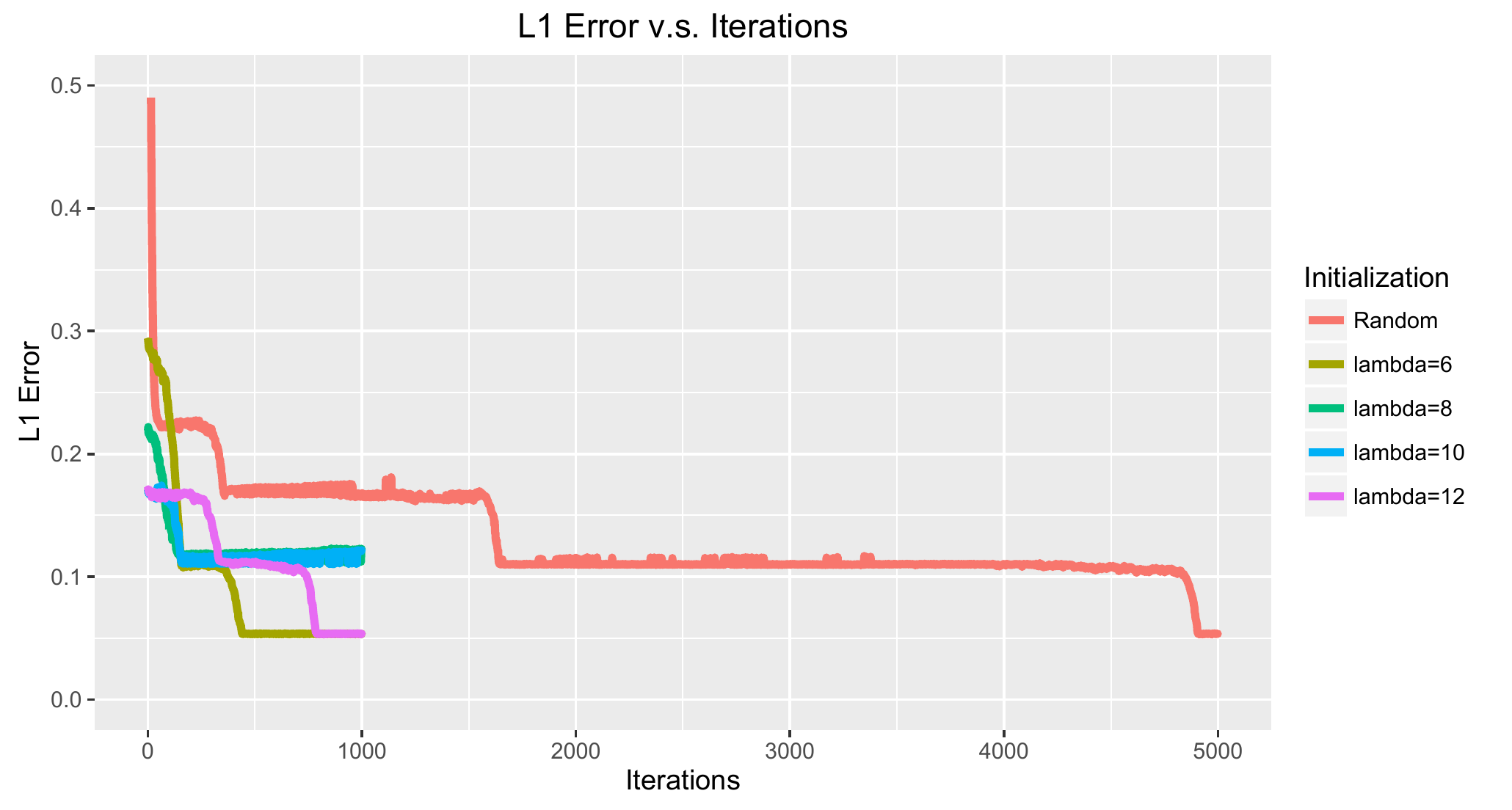}
\caption{The evolution of topic reconstruction $\ell_1$ errors of Gibbs sampler with different initializations: ``Random'' means random initialization, and  ``lambda=6'' means initializing with the assignment earned using \texttt{Word+Refine} algorithm with $\lambda=6$ (best viewed in color).}
\label{fig:l1-trend}
\end{figure}

\begin{figure}
\centering
\includegraphics[width=.8\textwidth]{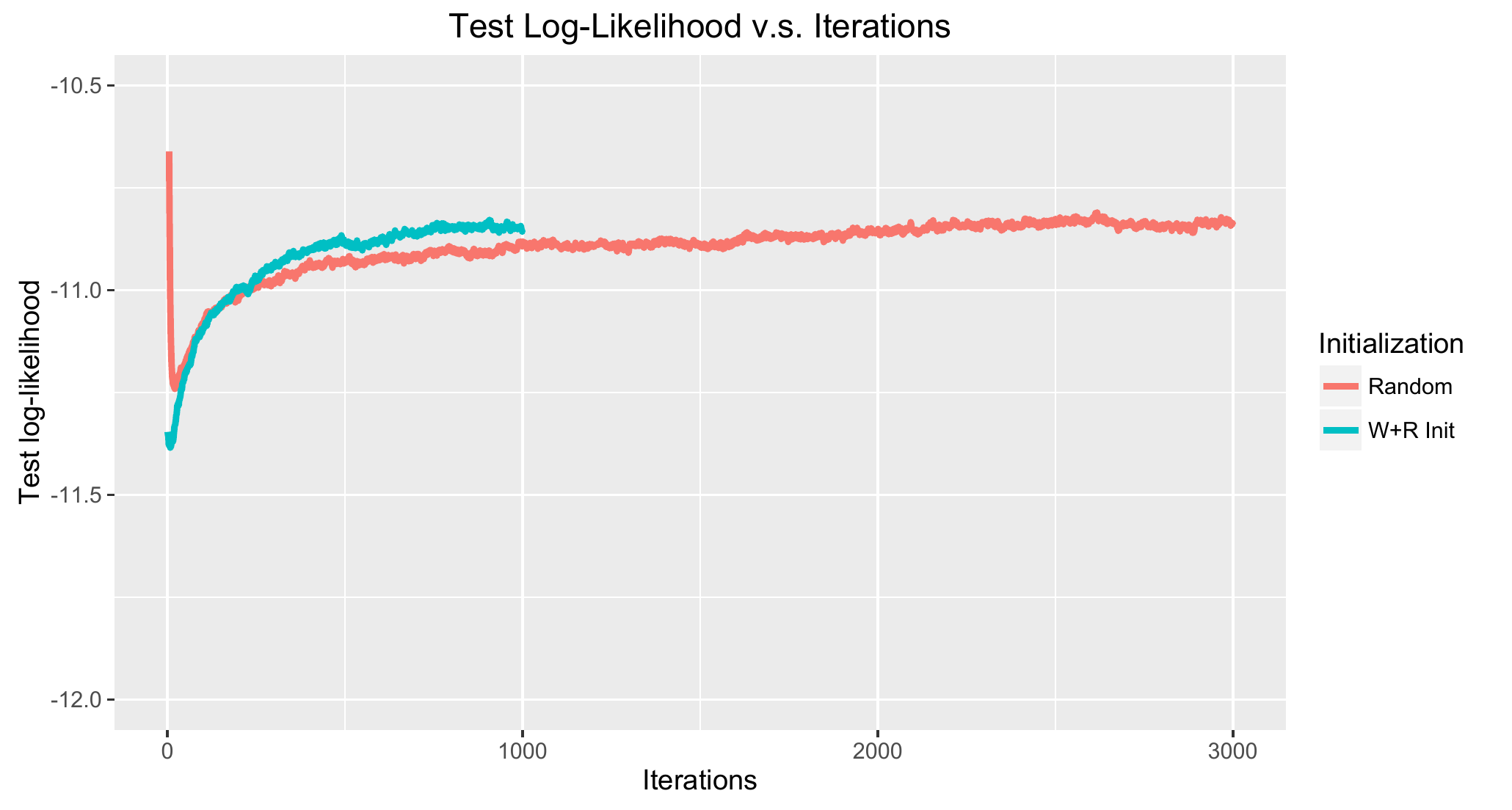}
\caption{The evolution of the predictive word log-likelihood of the Enron dataset with different initializations: ``Random'' means random initialization, and ``W+R Init'' means initializing with the assignment learned using \texttt{Word+Refine} algorithm.}
\label{fig:heldout-trend}
\end{figure}

\begin{table}
\centering
\begin{tabular}{| l || c | c | c || c | c | c || c | c | c |}
\hline
\multirow{2}{*}{Enron} & \multicolumn{3}{c||}{$\beta=0.1$}  & \multicolumn{3}{c||}{$\beta=0.01$} & \multicolumn{3}{c|}{$\beta=0.001$} \\ \cline{2-10}
 & hard & original & KL & hard & original & KL & hard & original & KL \\ \hline
 \texttt{CGS} & -5.932 & -8.583 & 3.899 & -5.484 & -10.781 & 7.084 & -5.091 & -13.296 & 10.000 \\ \hline
 \texttt{VB} & -6.007 & -8.803 & 4.528 & -5.610 & -10.202 & 7.010 & -5.334 & -11.472 & 9.010 \\ \hline
 \texttt{W+R} & -5.434 & -9.843 & 4.541 & -5.147 & -11.673 & 7.225 & -4.918 & -13.737 & 9.769 \\ \hline
 \multirow{2}{*}{NYT} & \multicolumn{3}{c||}{$\beta=0.1$}  & \multicolumn{3}{c||}{$\beta=0.01$} & \multicolumn{3}{c|}{$\beta=0.001$} \\ \cline{2-10}
 & hard & original & KL & hard & original & KL & hard & original & KL \\ \hline
 \texttt{CGS} & -6.594 & -9.361 & 4.374 & -6.205 & -11.381 & 7.379 & -5.891 & -13.716 & 10.135 \\ \hline
 \texttt{VB} & -6.470 & -10.077 & 5.666 & -6.269 & -11.509 & 7.803 & -6.023 & -12.832 & 9.691 \\ \hline
 \texttt{W+R} & -6.105 & -10.612 & 5.059 & -5.941 & -12.225 & 7.315 & -5.633 & -14.524 & 9.939 \\ \hline
\end{tabular}
\caption{The predictive word log-likelihood on new documents for Enron ($K=100$ topics) and NYTimes ($K=100$ topics) datasets with fixed $\alpha$ value. ``hard'' is short for hard predictive word log-likelihood which is computed using the word-topic assignments inferred by the \texttt{Word} algorithm, ``original'' is short for original predictive word log-likelihood which is computed using the document-topic distributions inferred by the sampler, and ``KL'' is short for symmetric KL-divergence.}
\label{tab:enron}
\end{table}

\textbf{Topic Reconstruction Error.} In the main text, we observed that the \texttt{Anchor} method is the most competitive with \texttt{Word+Refine} on larger synthetic data sets, but that \texttt{Word+Refine} still outperforms \texttt{Anchor} for these larger data sets.  We found this to be true as we scale up further; for instance, for 20,000 documents from the SynthB data, \texttt{Anchor} achieves a topic reconstruction score of 0.103 while \texttt{Word+Refine} achieves 0.095.

\textbf{Log likelihood comparisons on real data.}
Table~\ref{tab:enron} contains further predictive log-likelihood results on the Enron and NYTimes data sets.  Here we also show results on VB, which also indicate (as expected) that our approach does well with respect to the hard log-likelihood score but less well on the original log-likelihood score. We omit the results of the Anchor method since it cannot adjust its result on different combinations of $\alpha$ and $\beta$ values \footnote{We also observed that there are 0 entries in the learned topic matrix, which makes it difficult to compute the predictive log-likelihood.}. In Figure \ref{fig:heldout-trend}, we show the evolution of predictive heldout log-likelihood of the Gibbs sampler initialized with the \texttt{Word+Refine} optimized assignment for 3 iterations for the Enron dataset. With these semi-optimized initializations, we also observed significant speed-up compared to random initializations.

\end{appendices}

\end{document}